\documentclass[oneside,11pt]{article} 
\usepackage{graphicx} 
\usepackage[utf8]{inputenc} 
\usepackage[T1]{fontenc}    
\usepackage{microtype}   

\usepackage{xcolor}
\usepackage{amsmath}
\usepackage{amssymb}
\usepackage{amsthm}
\usepackage{mathtools}
\usepackage{commath}
\usepackage{dsfont}

\theoremstyle{definition}
\newtheorem{theorem}{Theorem}[section]
\newtheorem{lemma}[theorem]{Lemma}
\newtheorem{definition}[theorem]{Definition}
\newtheorem{claim}[theorem]{Claim}
\newtheorem{corollary}[theorem]{Corollary}

\title{Contrastive Learning with Nasty Noise}
\author{Ziruo Zhao \\
zzhao83@stevens.edu}
\date{\today}

\begin{document}

\maketitle

\begin{abstract}
Contrastive learning has emerged as a powerful paradigm for self-supervised representation learning. This work analyzes the theoretical limits of contrastive learning under nasty noise, where an adversary modifies or replaces training samples. Using PAC learning and VC-dimension analysis, lower and upper bounds on sample complexity in adversarial settings are established. Additionally, data-dependent sample complexity bounds based on the $\ell_2$-distance function are derived.
\end{abstract}

\section{Preliminaries}

In contrastive learning, for samples from a domain $\mathrm{V}$, the learned representation is a distance $\rho :\mathrm{V}\times \mathrm{V} \rightarrow \mathbb{R}$. A popular distance function is ${\ell}_p$-norm ${\rho}_p(x,y)=\norm{f(x)-f(y)}_p$ for some representation function $f:\mathrm{V} \rightarrow \mathbb{R}^d$ in dimension $d$. 
Let  $\mathcal{H}$ be a hypothesis class and $h_{\rho} \in \mathcal{H}$ be an hypothesis with respect to an unknown distance function $\rho$. A learning task is specified using a hypothesis class of boolean classifiers defined over an instance space, denoted $\mathrm{V}$. A boolean classifier is a function $h: \mathrm{V^3} \rightarrow \{0,1\}$.


\subsection{The classical PAC model}
\paragraph{} In PAC model, the learning algorithm has access to labeled examples of the form $(x,y^+,z^-)$ from a distribution $\mathcal{D}$. The examples are labeled by the target classifier $h^*$ in $\mathcal{H}$. The goal of contrastive learning is to create a classifier from $\mathcal{H}$ which accurately labels subsequent unlabeled inputs. For a distance function $\rho$, the hypothesis with respect to this distance function is $h_{\rho}(x,y,z)=\mathbf{sign}(\rho(x,y)-\rho(x,z))$. When $h_{\rho}(x,y,z)=-1$, the example will be labeled as $(x,y^+,z^-)$, meaning $\rho (x,y) < \rho (x,z)$; otherwise $h_{\rho}(x,y,z)=1$ then label $(x,y^-,z^+)$. 




\begin{definition}[Contrastive learning, classical PAC case]
    A hypothesis class $\mathcal{H}$ is PAC learnable if there exist a learning algorithm that, for any $h^*\in \mathcal{H}$, any input parameters $0< \epsilon <1/2$ and $0< \delta <1$ and any distribution $\mathcal{D}$, when given access to samples in the form $(x,y^+,z^-)$, the minimum number of samples required is denoted as $n(\epsilon,\delta)$, with probability at least $1-\delta$ outputs a function $h_{\rho} \in \mathcal{H}$ to achieve error rate $\epsilon$, that is
\begin{equation*}
    Pr_{(x,y,z)\sim \mathcal{D}}[h_{\rho}(x,y,z)\neq h^*(x,y,z)] <\epsilon
\end{equation*}
\end{definition}


\subsection{PAC learning with nasty noise}
\paragraph{} As in PAC model, a distribution $\mathcal{D}$ and a target classifier is given. For PAC model in the presence of nasty noise, the adversary draws a sample set $S'$ of size $n$ from the distribution $\mathcal{D}$. Having full knowledge of the learning algorithm, the target classifier $h^*$, the distribution $\mathcal{D}$, and the sample drawn, the adversary chooses $m$ examples from the sample set. The $m$ examples chosen are modified or replaced by the adversary, introducing noisy views or noisy labels. In contrastive learning, noisy views compel the representations of different views to align with each other, even if there is no useful information for learning distinguishable features. This misalignment ultimately causes the algorithm to learn inappropriate noisy features.

The $n-m$ examples not chosen by the adversary remain unchanged and are labeled by their correct labels according to $h^*$. The modified sample of n points, denoted S, is then given to the learning algorithm. The number of examples $m$ that the adversary may modify should be distributed according to the binomial distribution $\mathbf{Bin}(n,\eta)$, where $\eta$ is the rate of nasty noise.

\begin{definition}[Contrastive learning, nasty noise case]
    Let $\mathcal{H}$ be a hypothesis class, and $h^*$ be a target classifier from $\mathcal{H}$. A nasty adversary takes as input a sample size $n$ requested by the learning algorithm, draws $n$ samples from a distribution $\mathcal{D}$ and labels samples according to $h^*$. Denote the rate of nasty noise as $\eta$. The adversary may replace $m$ examples from the sample set, where $m\sim \mathbf{Bin}(n,\eta)$, and returns the modified sample set $S$ to the learning algorithm. With accuracy parameter $0<\epsilon <1$ and confidence parameter $0<\delta <1$, We say that an algorithm $\mathcal{A}$ PAC learns the class $\mathcal{H}$ with nasty sample noise of rate $\eta \geq 0$ if it outputs a function $h_{\rho} \in \mathcal{H}$ to achieve error rate $\epsilon$, with probability at least $1-\delta$.
\end{definition}

\subsection{VC theory basics}
\begin{definition}[Shattering]
    Let $\mathcal{X}$ be an instance space. We say that a finite set $S\subset \mathcal{X}$ is shattered by a hypothesis class $\mathcal{H}$ if, for each of the $2^{|S|}$ possible labeling of the points in $S$, there exists some function in $\mathcal{H}$ consistent with that labeling.
\end{definition}

\begin{definition}[VC-dimension]
    The VC-dimension of a hypothesis class $\mathcal{H}$, denoted $VCdim(\mathcal{H})$, is the maximal size $d$ of a set $S\subset \mathcal{X}$ that can be shattered by $\mathcal{H}$. If $\mathcal{H}$ can shatter sets of any integer $d$, we say $\mathcal{H}$ has infinite VC-dimension that $VCdim(\mathcal{H})=\infty$.
\end{definition}

\begin{lemma}
    For any two classes $\mathcal{H}$ and $\mathcal{F}$ over $\mathcal{X}$, 
\begin{enumerate}
    \item The class of negations $\{h|\mathcal{X}\backslash h\in \mathcal{H}\}$ has the same VC-dimension as the class $\mathcal{H}$
    \item The class of unions $\{h\cup f|h\in \mathcal{H},f\in \mathcal{F}\}$ has the VC-dimension at most $VCdim(\mathcal{H})+VCdim(\mathcal{F})+1$.
    \item The class of intersections $\{h\cap f|h\in \mathcal{H},f\in \mathcal{F}\}$ has the VC-dimension at most $VCdim(\mathcal{H})+VCdim(\mathcal{F})+1$.
\end{enumerate}
\end{lemma}

\begin{definition}
    The dual $\mathcal{H}^\bot \subseteq \{0,1\}^\mathcal{H}$ of a class $\mathcal{H} \subseteq \{0,1\}^\mathcal{X}$ is defined to be the set $\{x^\bot |x\in \mathcal{X}$ where $x^\bot$ is defined by $x^\bot (h)=h(x)$ for all $h\in \mathcal{H}$.
\end{definition}

The following claim gives a tight bound on the VC-dimension of the dual class:
\begin{claim}
    For every class $\mathcal{H}$, $VCdim(\mathcal{H}) \geq \lfloor log\ VCdim(\mathcal{H}^\bot)\rfloor$.
\end{claim}

In the following discussion, we restrict our focus to finite-dimensional hypothesis classes. In this paper, the main use of the VC-dimension is in consisting $\alpha$-samples.

\begin{definition}[$\alpha$-sample]
    A set of points $S\subset \mathcal{X}$ is an $\alpha$-sample for the hypothesis class $\mathcal{H} \subseteq \{0,1\}^\mathcal{X}$ under the distribution $\mathcal{D}$ over $\mathcal{X}$, if it holds that for every $h\in \mathcal{H}$:
\begin{equation*}
    \left |\mathcal{D}(h)-\frac{|S\cap h|}{|S|}\right |\leq \alpha.
\end{equation*}
\end{definition}

\begin{theorem}
    There is a constant $c$, such that for any class $\mathcal{H} \subseteq \{0,1\}^\mathcal{X}$ of VC-dimension $VCdim(\mathcal{H})$, and distribution $\mathcal{D}$ over $\mathcal{X}$, and any $\alpha >0$, $\delta >0$, if 
\begin{equation*}
    n \geq \frac{c}{\alpha ^2}\left (VCdim(\mathcal{H})+log\frac{1}{\delta}\right )
\end{equation*}
    examples are drawn i.i.d. from $\mathcal{X}$ according to the distribution $\mathcal{D}$, they constitute an $\alpha$-sample for $\mathcal{H}$ with probability at least $1-\delta$.
\end{theorem}

\begin{definition}[Natarajan dimension]
    Let $\mathcal{X}$ be an instance space, $\mathcal{Y}$ be the set of labels, and let $\mathcal{H}\subseteq \mathcal{Y}^\mathcal{X}$. We say that a set $S\subset \mathcal{X}$ is N-shattered by $\mathcal{H}$ if there exist $f_1$, $f_2:\mathcal{X}\rightarrow \mathcal{Y}$ such that $f_1(x)\neq f_2(x)$ for all $x\in S$ and for every $B \subset S$ there exists $g\in \mathcal{H}$ such that:
\begin{equation*}
    g(x)=f_1(x)\ for\ x\in B\ and\ g(x))=f_2(x)\ for\ x\notin B
\end{equation*}
The Natarajan dimension $Ndim(\mathcal{H})$ is the maximal size of a N-shattered set $S\subset \mathcal{X}$.
\end{definition}

\begin{lemma}
    If $|S|$ is finite, then for the sample complexity $n(\epsilon,\delta)$ of the PAC case it holds that:
\begin{equation*}
    n(\epsilon,\delta)=O\left (\frac{Ndim(\mathcal{H})\text{log}|S|}{\epsilon}\text{polylog}\left (\frac{1}{\epsilon},\frac{1}{\delta}\right )\right )\ and\ \Omega \left (\frac{Ndim(\mathcal{H})}{\epsilon}\text{polylog}\left (\frac{1}{\epsilon},\frac{1}{\delta}\right )\right )
\end{equation*}
VC-dimension is a special case of Natarajan dimension when $|S|=2$.
\end{lemma}

\section{Lower Bound}
\subsection{Lower bound in classical PAC case}
\begin{theorem}[Lower bound for arbitrary distances]
    For an arbitrary distance function $\rho$ and a dataset of size $N$, the sample complexity of contrastive learning is $n(\epsilon,\delta)=\Omega \left (\frac{N^2}{\epsilon}\text{polylog}\left (\frac{1}{\epsilon},\frac{1}{\delta}\right )\right )$ in the classical PAC case.
\end{theorem}

\begin{proof}
Consider a graph that has a vertex representing each element in the dataset. Let the set of vertices be $V=\{v_1,...,v_N\}$. Let $S$ be the set of all possible three-element combination consisting of vertices in $V$:
\begin{equation*}
    S={\cup}_{i\in [N]}\{(v_i,v_{i+1},v_{i+2}), (v_i,v_{i+2},v_{i+3}),...,(v_i,v_{N-1},v_N)\}
\end{equation*}
If there exist a classifier can correctly label all pairs in $S$ then changing the anchor within any pair in $S$ should not affect its accuracy to correctly label the pair. So this classifier performs well on any arbitrary query from dataset $V$.

\vspace{5pt}
Then we prove that the set of samples $S$ is shattered. Let $h^*_\rho$ be a true classifier for S. For every $i\in [N]$, we define a graph $G_i=(V_i,E_i)$ where $V_i=\{v_{i+1},...,v_N\}$ and $E_i$ contains a directed edge $(v_j,v_{j+1})$ for each query $(v_i,v_j,v_{j+1})$ according to $h^*_\rho$. For example of labeling $(v_i,v_j^+,v_{j+1}^-)$, the direction of the edge between $(v_j,v_{j+1})$ is from positive $v_j$ to negative $v_{j+1}$. The graph $G_i$ is acyclic since it is an orientation of a path. Then we can topologically sort $G_i$ to obtain some topological order $p_{i+1}^{(i)},...,p_N^{(i)}$ for the vertices $v_{i+1},...,v_N$. Consider a distance function defined as $\rho (v_i,v_j):=n+p_j^{(i)}$ for all $i<j\leq n$. Thus, all distances are in the range $[N,2N]$. Therefore, this is a metric since triangle inequalities are satisfied.

\vspace{5pt}
For $i,j,k$ such that $i<j,k$ and $|j-k|=1$, $v_i,v_j,v_k$ will form a pair with anchor $v_i$ in $S$. $Y_S$ will contain a labeling for this pair, if $h^*_\rho$ give the label as $(v_i,v_j^+,v_{k}^-)$ then the edge between $(v_j,v_k)$ is from $v_j$ to $v_k$, accordingly $p_k^{(i)}>p_j^{(i)}$ which indicates $\rho (v_i,v_k)>\rho (v_i,v_j)$, vice versa. Therefore, we say that this distance function satisfies all the samples, the set of samples $S$ is shattered. An $\Omega(N^2)$ lower bound on the VC-dimension follows. By applying Lemma 2.11, we have the lower bound of sample complexity $n(\epsilon,\delta)=\Omega \left (\frac{N^2}{\epsilon}\text{polylog}\left (\frac{1}{\epsilon},\frac{1}{\delta}\right )\right )$.
\end{proof}

\begin{theorem}[Lower bound for $\ell_p$-distances]
    For any real constant $p\in (0,\infty)$, a dataset $V$ of size $N$, and the ${\ell}_p$ distance ${\rho}_p:\mathrm{V}\times \mathrm{V} \rightarrow \mathbb{R}$ in a $d$-dimensional space, the sample complexity of contrastive learning is $n(\epsilon,\delta)=\Omega \left (\frac{\text{min}(Nd,N^2)}{\epsilon}\text{polylog}\left (\frac{1}{\epsilon},\frac{1}{\delta}\right )\right )$ in the classical PAC case.
\end{theorem}

\begin{proof}
As shown in theorem 3.1, for any distance function, the lower bound of VC-dimension is $\Omega (N^2)$. Then we discuss the case that $d<N$, and will show that a sample set of size $\Omega (Nd)$ can be shattered. Let $V$ be the dataset of size $N$. We construct a sample set S with $N-d$ anchors and $d$ other points with $d\geq 2$. We denote the dataset as $V=\{x_1,x_2,...,x_{N-d},y_1,y_2,...,y_d\}$ where $x_i$'s represent anchors and $y_j$'s represent other points. The query set is defined as following:
\begin{equation*}
    S={\cup}_{i\in [N-d]}\{(x_i,y_1,y_2), (x_i,y_1,y_3),...,(x_i,y_1,y_d)\}
\end{equation*}
There are $(d-1)(n-d)$ samples in the set $S$. Recall that the distance function is ${\ell}_p$-norm ${\rho}_p(x,y)=\norm{f(x)-f(y)}_p$ for some representation function $f:\mathrm{V} \rightarrow \mathbb{R}^d$ in dimension $d$. The corresponding hypothesis is $h_p(x,y,z)=\mathbf{sign}\left (\norm{f(x)-f(y)}_p-\norm{f(x)-f(z)}_p\right )$. 

\vspace{5pt}
Next, we define a representation function which can make the hypothesis $h_p$ satisfy labeling of queries. For points $y_j$, let the j-th coordinate of $f(y_j)$ equal to $1$, other coordinates equal to $0$. For anchors $x_i$, let the first coordinate of $f(x_i)$ be $\frac{1}{2}$. For $j\in \{2,...,d\}$, $f(x_i)_j=0$ if $(x_i,y_1,y_j)$ is labeled as $(x_i,y_1^+,y_j^-)$, otherwise $f(x_i)_j=1$ if $(x_i,y_1^-,y_j^+)$. Using this representation function $f$, $h_p\circ f$ will follow:
\begin{equation*}
    \norm{f(x_i)-f(y_1)}_p^p-\norm{f(x_i)-f(y_j)}_p^p=
    \begin{cases}
        ({\frac{1}{2}})^p-(({\frac{1}{2}})^p+1^p)<0 & if (x_i,y_1^+,y_j^-)\\
        ({\frac{1}{2}})^p-(({\frac{1}{2}})^p-1^p)>0 & if (x_i,y_1^-,y_j^+)
    \end{cases}
\end{equation*}
Hence, $h_p\circ f$ can satisfy the labeling of all queries. The sample set $S$ can be shattered. For the case that $d=1$, any set of $\lfloor \frac{n}{3}\rfloor$ disjoint queries can be satisfied, gives $\Omega (N)$ lower bound. Therefore, the lower bound of VC-dimension is $\Omega (Nd)$ when $d<N$.
\end{proof}

\subsection{Lower bound in nasty noisy case}
\begin{theorem}
    Let $\mathcal{H}$ be a non-trivial hypothesis class, $\eta$ be a noise rate. Given access to samples of the form $(x,y^+,z^-)$ from a distribution $\mathcal{D}$ and then corrupted by a nasty adversary, for any $\epsilon <2\eta$, $\delta <\frac{1}{2}$, there is no algorithm that learns a classifier from $\mathcal{H}$ with error $\epsilon$ with probability at least $1-\delta$.
\end{theorem}

\begin{proof}
Consider there are two classifiers $h_1,h_2\in \mathcal{H}$ such that $\text{Pr}_\mathcal{D}[h_1(x,y,z)\neq h_2(x,y,z)]=2\eta$. Assume that an adversary can force the labeled examples shown to the learning algorithm to be identically distributed whether $h_1$ or $h_2$ is the target classifier. 
Assume $(x_1,y_1,z_1),(x_2,y_2,z_2)$ be two samples from distribution D that satisfy $h_1(x_1,y_1,z_1)=h_2(x_1,y_1,z_1)=(x_1,y_1^+,z_1^-)$ and $h_1(x_2,y_2,z_2)=(x_2,y_2^+,z_2^-)\neq h_2(x_2,y_2,z_2)=(x_2,y_2^-,z_2^+)$. We define the distribution D to be $D(x_1,y_1,z_1)=1-2\eta$, $D(x_2,y_2,z_2)=2\eta$ and $D(x,y,z)=0$ for all other samples. This indicates that $\text{Pr}_\mathcal{D}[h_1(x,y,z)\neq h_2(x,y,z)]=\text{Pr}_\mathcal{D}[(x_2,y_2,z_2)]=2\eta$. 

\vspace{5pt}
Then we describe a nasty adversary strategy. Let $n$ be the size of samples required by the learning algorithm. First, adversary i.i.d draws $n$ samples from the distribution $\mathcal{D}$ and labels them according to the target classifier. Then, for each occurrence of $(x_1,y_1,z_1)$, the adversary remain it unchanged, while for each occurrence of $(x_2,y_2,z_2)$, the adversary flips the label with probability $\frac{1}{2}$. The modified samples indeed according to the binomial distribution $Bin(n,\eta)$, since distribution $\mathcal{D}$ is known to the adversary and $\text{Pr}[\text{Flip label of }(x_2,y_2,z_2)]=2\eta \cdot \frac{1}{2}=\eta$. The modified sample set $S$ is given to the learning algorithm. 

\vspace{5pt}
Whether the target classifier is $h_1$ or $h_2$, labeled samples in the given set $S$ are distributed according to the following distribution:
\begin{equation*}
    \begin{cases}
        \text{Pr}[(x_1,y_1^+,z_1^-)]=1-2\eta,\\
        \text{Pr}[(x_2,y_2^+,z_2^-)]=\eta,\\
        \text{Pr}[(x_2,y_2^-,z_2^+)]=\eta.\\
    \end{cases}
\end{equation*}
Therefore, based on this sample set, it is impossible to differentiate between $h_1$ and $h_2$, no algorithm can learn better than a random guess.
\end{proof}

\begin{theorem}
    For any non-trivial hypothesis class $\mathcal{H}$, any noise rate $\eta >0$, confidence parameter $0<\delta <\frac{1}{342}$ and $0<\Delta <\frac{1}{12}\eta$, the sample size needed of contrastive learning with accuracy $\epsilon =2\eta +\Delta$ in the presence of nasty noise with noise rate $\eta$ is $\Omega (\frac{\eta}{{\Delta}^2})$.
\end{theorem}

\begin{proof}
Consider a similar case shown in proof of Theorem 3.3. There are two classifier $h_1,h_2\in \mathcal{H}$. Let $(x_1,y_1,z_1),(x_2,y_2,z_2)$ be two samples from distribution D that satisfy $h_1(x_1,y_1,z_1)=h_2(x_1,y_1,z_1)=(x_1,y_1^+,z_1^-)$ and $h_1(x_2,y_2,z_2)=(x_2,y_2^+,z_2^-)\neq h_2(x_2,y_2,z_2)=(x_2,y_2^-,z_2^+)$. We define the distribution D to be $D(x_1,y_1,z_1)=1-\epsilon$, $D(x_2,y_2,z_2)=\epsilon$ and $D(x,y,z)=0$ for all other samples. The target classifier $h^*$ can be either $h_1$ or $h_2$. 

\vspace{5pt}
The nasty adversary strategy is also similar. Let $n$ be the size of samples required by the learning algorithm. First, adversary i.i.d draws $n$ samples from the distribution $\mathcal{D}$ and labels them according to the target classifier. Then, for each occurrence of $(x_1,y_1,z_1)$, the adversary remain it unchanged, while for each occurrence of $(x_2,y_2,z_2)$, the adversary flips the label with probability $\frac{\eta}{\epsilon}$. The modified samples indeed according to the binomial distribution $Bin(n,\eta)$, since distribution $\mathcal{D}$ is known to the adversary and $\text{Pr}[\text{Flip label of }(x_2,y_2,z_2)]=\epsilon \cdot \frac{\eta}{\epsilon}=\eta$. The modified sample set $S$ is given to the learning algorithm. We denote $h^*((x_2,y_2,z_2))=(x_2,y_2^+,z_2^-)$, labeled samples in the given set $S$ are distributed according to the following distribution:
\begin{equation*}
    \begin{cases}
        \text{Pr}[(x_1,y_1^+,z_1^-)]=1-\epsilon,\\
        \text{Pr}[(x_2,y_2^+,z_2^-)]=\epsilon \cdot \left (1-\frac{\eta}{\epsilon}\right )=\epsilon -\eta=\eta +\Delta,\\
        \text{Pr}[(x_2,y_2^-,z_2^+)]=\epsilon \cdot \frac{\eta}{\epsilon}=\eta.\\
    \end{cases}
\end{equation*}
We will show that an algorithm that create a classifier from $\mathcal{H}$ with accuracy $\epsilon$ using samples drawn from distribution $\mathcal{D}$ and the size of samples is $m<\frac{17\eta (1-\eta)}{37{\Delta}^2}$. Let $A$ be the algorithm. $A$ outputs a classifier $h$ using samples of size $n$. We denote the expected error of $h$ as $\text{err}_A(n)$. Let $B$ be the Bayes strategy: if majority label of $(x_2,y_2,z_2)=(x_2,y_2^+,z_2^-)$, outputs $h_1$; otherwise, outputs $h_2$. We denote the expected error of the output classifier as $\text{err}_B(n)$. Since $B$ minimizes the probability of choosing the wrong classifier, we have $\text{err}_B(n)\leq \text{err}_A(n)$ for all $m$. 

\vspace{5pt}
Then we define two events over runs of B: Let $M$ be the number of sample $(x_2,y_2,z_2)$ shown in the sample set of size $n$. Let $m$ be the number of samples which are corrupted by the adversary. $\text{BAD}_1$ is the event that $m\geq \lceil \frac{M}{2}\rceil +1$. $\text{BAD}_2$ is the event that $M\leq \frac{36\eta (\eta +\Delta)}{37{\Delta}^2}$. Clearly if $\text{BAD}_1$ happens, the majority will be the modified label of $(x_2,y_2,z_2)$, so B will output wrong classifier. To give the bound of $\text{Pr}[\text{BAD}_1]$, we will show lower bounds for $\text{Pr}[\text{BAD}_1|\text{BAD}_2]$ and $\text{Pr}[\text{BAD}_2]$ respectively.

\paragraph{$\text{Pr}[\text{BAD}_2]$}
Note that $M$ is a random variable distributed by the binomial distribution with parameters $n$ and $\epsilon$, i.e.$M\sim \mathbf{Bin}(n,\epsilon)$. The expected value of $M$ is $E[M]=n\epsilon$. Recall that $n$ is upper bounded by $\frac{17\eta (1-\eta)}{37{\Delta}^2}$. We have that
\begin{align*}
    \text{Pr}[\text{BAD}_2] &= \text{Pr}[M\leq \frac{36\eta (\eta +\Delta)}{37{\Delta}^2}]\\
    &\geq \text{Pr}[M\leq \frac{18\eta\epsilon}{37{\Delta}^2}]\\
    &\geq \text{Pr}[M\leq \frac{18}{17}n\epsilon]\\
    &=\text{Pr}[\frac{M}{E[M]}-1\leq \frac{1}{17}].
\end{align*}
Now we assume that $n\geq \frac{51}{\epsilon}$. Thus, using Chernoff's inequality, we have the lower bound of $\text{Pr}[\text{BAD}_2]$: 
\begin{align*}
    \text{Pr}[\text{BAD}_2] \geq 1-e^{-1/17}.
\end{align*}

\paragraph{$\text{Pr}[\text{BAD}_1|\text{BAD}_2]$}
For now we assume that $\text{BAD}_2$ holds, namely that $M\leq \frac{36\eta (\eta +\Delta)}{37{\Delta}^2}$. And our proof uses the following result in probability theory:
\begin{claim}
    Let $S_{N,p}$ be a random variable distributed by the binomial distribution with parameters $N$ and $p$, and let $q=1-p$. For all $N>\frac{37}{pq}$:
    \begin{align*}
        &\text{Pr}[S_{N,p}\geq \lfloor Np\rfloor +\lfloor \sqrt{Npq-1}\rfloor]>\frac{1}{19},\\
        &\text{Pr}[S_{N,p}\leq \lceil Np\rceil -\lceil \sqrt{Npq-1}\rceil]>\frac{1}{19}.
    \end{align*}
\end{claim}

By Claim 3.5 with $m\sim \mathbf{Bin}(M,\frac{\eta}{\epsilon})$, if $M\geq \frac{37(2\eta +\Delta)^2}{\eta(\eta +\Delta)}$, we have
\begin{equation*}
    \text{Pr}\left [m\geq \left \lfloor M\frac{\eta}{2\eta +\Delta}\right \rfloor +\left \lfloor \sqrt{M\frac{\eta (\eta +\Delta)}{(2\eta +\Delta)^2}-1}\right \rfloor \right ]>\frac{1}{19}
\end{equation*}
then we consider the following inequality:
\begin{equation*}
    \left \lfloor M\frac{\eta}{2\eta +\Delta}\right \rfloor +\left \lfloor \sqrt{M\frac{\eta (\eta +\Delta)}{(2\eta +\Delta)^2}-1}\right \rfloor \geq \left \lceil \frac{M}{2}\right \rceil +1
\end{equation*}
which is implied by
\begin{equation*}
    M\frac{\eta}{2\eta +\Delta}+\sqrt{M\frac{\eta (\eta +\Delta)}{(2\eta +\Delta)^2}-1} \geq \frac{M}{2}+3,
\end{equation*}
which is implied by the following two conditions:
\begin{align*}
    &\frac{1}{2}\sqrt{M\frac{\eta (\eta +\Delta)}{(2\eta +\Delta)^2}-1} \geq 3,\\
    &\frac{1}{2}\sqrt{M\frac{\eta (\eta +\Delta)}{(2\eta +\Delta)^2}-1} \geq \frac{1}{2}M\frac{\Delta}{2\eta +\Delta}.
\end{align*}

These two conditions holds if we assume $\frac{37(2\eta +\Delta)^2}{\eta(\eta +\Delta)}\leq M\leq \frac{36\eta (\eta +\Delta)}{37{\Delta}^2}$. With $M$ in this range, it follows that
\begin{equation*}
    \text{Pr}\left [m\geq \left \lceil \frac{M}{2}\right \rceil +1 \right ]=\text{Pr}[\text{BAD}_1|\text{BAD}_2]>\frac{1}{19}
\end{equation*}

Now, we consider the lower bound of $M$ may be removed. Since $m\sim \mathbf{Bin}(M,\frac{\eta}{\epsilon})$, with lower M, the probability that m will be at least $\lceil \frac{M}{2} \rceil +1$ becomes higher. Therefore, the lower bound of $M$ can be safely removed. We have $\text{Pr}[\text{BAD}_1]\geq \text{Pr}[\text{BAD}_1|\text{BAD}_2]\cdot \text{Pr}[\text{BAD}_2]\geq \frac{1}{19}\cdot (1-e^{-1/17})>\frac{1}{342}$ under assumption of $n\geq \frac{51}{\epsilon}$.

\vspace{5pt}
Next, we consider that the lower bound of $n$ may be removed. By contradiction, firstly we assume the lower bound of $n$ is required. This means that there exists an algorithm $A_1$ which can learn the class $\mathcal{H}$ with accuracy $\eta +2\Delta$ with probability at least $1-\delta$ under noise rate $\eta$, using samples of size $n_1\leq \frac{51}{\epsilon}$. In addition, we have $\frac{51}{\epsilon}<\frac{17\eta (1-\eta)}{37{\Delta}^2}$ by the condition on $\Delta$. We assume that there exists an algorithm $A_2$ using samples of size $\frac{51}{\epsilon}<n_2<\frac{17\eta (1-\eta)}{37{\Delta}^2}$ works as follows: $A_2$ randomly draw $n_1$ samples from original sample set of $n_2$, and feeds them to $A_1$. Then $A_2$ contradicts our discussion above. Thus, no such algorithm $A_1$ may exist, and so the lower bound of $n$ can be safely removed. 

\vspace{5pt}
Finally, we have that $\text{Pr}[\text{err}_A(n)>\epsilon]\geq \text{Pr}[\text{err}_B(n)>\epsilon]\geq \text{Pr}[\text{BAD}_1]>\frac{1}{342}$ under assumption $n<\frac{17\eta (1-\eta)}{37{\Delta}^2}$. This indicates the lower bound of sample complexity shown in Theorem 3.4.
\end{proof}

\begin{theorem}
    For any non-trivial hypothesis class $\mathcal{H}$ with $VCdim\geq 3$, any $0<\epsilon \leq \frac{1}{8}$, $0<\delta <\frac{1}{12}$ and $0<\Delta <\epsilon$, the sample size needed of contrastive learning with accuracy $\epsilon$ in the presence of nasty noise with noise rate $\eta =\frac{1}{2}(\epsilon -\Delta)$ is $\Omega (\frac{VCdim}{{\Delta}})$.
\end{theorem}

\begin{proof}
Let V be a dataset. Let S be a set of samples in the form of $(x,y,z)$ where $x,y,z\in V$ shattered by a hypothesis class $\mathcal{H}$. We may assume that $\mathcal{H}$ is the power set of $S$. So we denote VC-dimension of $\mathcal{H}$ as $d$ in this proof, and let $S={(x_1,y_1,z_1),...,(x_d,y_d,z_d)}$.
Define a probability distribution $\mathcal{D}$ on the set $S$ as following:
\begin{equation*}
    \begin{cases}
        \mathcal{D}((x_1,y_1,z_1))=1-2\eta -8\Delta\\
        \mathcal{D}((x_2,y_2,z_2))=...=\mathcal{D}((x_{d-1},y_{d-1},z_{d-1}))=\frac{8\Delta}{d-2}\\
        \mathcal{D}((x_d,y_d,z_d))=2\eta
\end{cases}
\end{equation*}
We proof by contradiction. First assume that at most $n=\frac{d-2}{32\Delta}$ samples are used by the algorithm to learn $\mathcal{H}$. Then we define the nasty adversary strategy as follows: The adversary picks a target classifier $h^*$ uniformly at random from $\mathcal{H}$, then generates a sample set of size $n$ and label them by $h^*$. Then the adversary flips the label on each sample $(x_d,y_d,z_d)$ with probability $\frac{1}{2}$, while other samples remain unchanged. Then we show that there exists a target function for which the algorithm does not find $\epsilon$-accurate classifier with probability at least $\delta$. This is suffices to show that the algorithm create a classifier h such that $\text{Pr}_{\mathcal{D}}[h(x,y,z)\neq h^*(x,y,z)]$ with probability at least $1-\delta$. 

\vspace{5pt}
Let $\text{BAD}_1$ be the event that at least half of samples $(x_2,y_2,z_2),...(x_{d-1},y_{d-1},z_{d-1})$ are not drawn into the sample set. Thus, these samples are not seen by the algorithm. Given $\text{BAD}_1$, assume the set of $\frac{d-2}{2}$ unseen samples with lowest indices, and define $\text{BAD}_2$ as the event that h output by the algorithm misclassifies at least $\frac{d-2}{8}$ samples from the set of unseen samples. Finally, let $\text{BAD}_3$ be the event that $x_d$ is misclassified. Clearly,  $\text{BAD}_1 \wedge \text{BAD}_2 \wedge \text{BAD}_3$ implies that the classifier $h$ has error at least $\epsilon$. Since $h$ gives wrong labels on $\frac{d-2}{8}$ samples which have weight $\frac{8\Delta}{d-2}$ and on the sample $(x_d,y_d,z_d)$ which has weight $2\eta$, thus, the error of $h$ is at least $\Delta +2\eta =\epsilon$. Therefore, if an algorithm can learn the class with probability at least $1-\delta$, then $\text{Pr}[\text{BAD}_1 \wedge \text{BAD}_2 \wedge \text{BAD}_3]<\delta$ must hold. 

\paragraph{$\text{Pr}[\text{BAD}_1]$}
Denote $n_{seen}$ as the number of samples from $(x_2,y_2,z_2),...(x_{d-1},y_{d-1},z_{d-1})$ seen by algorithm. Since at most $\frac{d-2}{32\Delta}$ samples are given to the algorithm, the expected value of $n_{seen}$ is at most $(d-2)\cdot \frac{8\Delta}{d-2}\cdot \frac{d-2}{32\Delta}=\frac{d-2}{4}$. By Markov's inequality, with probability at least $\frac{1}{2}$, the number of seen samples is less than $\frac{d-2}{2}$, i.e.$\text{Pr}[n_{seen}\geq \frac{d-2}{2}]\leq \frac{E[n_{seen}]}{\frac{d-2}{2}}\leq \frac{1}{2}$. Therefore, $\text{Pr}[\text{BAD}_1]>\frac{1}{2}$.

\paragraph{$\text{Pr}[\text{BAD}_2|\text{BAD}_1]$}
Since the target $h^*$ is picked uniformly at random, hence $h$ will misclassify each unseen samples with probability $\frac{1}{2}$. Given $\text{BAD}_1$, the probability $\text{Pr}[\text{BAD}_2|\text{BAD}_1]$ is equivalent to: if we flips a fair coin $\frac{d-2}{2}$ times, the probability that the number of heads observed is at least $\frac{d-2}{8}$. 

\begin{claim}
    For every $0<\beta <\alpha \leq 1$, for every random variable $S\in [0,N]$ with $E[S]=\alpha N$, it holds that $\text{Pr}[S\geq \beta N]>\frac{\alpha -\beta}{1-\beta}$
\end{claim}
Using Claim 3.7, $\text{Pr}[\text{number of heads observed}\geq \frac{1}{4}\cdot \frac{d-2}{2}]>\frac{\frac{1}{2}-\frac{1}{4}}{1-\frac{1}{4}}=\frac{1}{3}$. Therefore, we have $\text{Pr}[\text{BAD}_2|\text{BAD}_1]>\frac{1}{3}$.

\paragraph{$\text{Pr}[\text{BAD}_3|\text{BAD}_1 \wedge \text{BAD}_2]$}
Recall that the adversary modified label of $(x_d,y_d,z_d)$ with probability $\frac{1}{2}$, thus, $\text{Pr}[\text{BAD}_3]=\frac{1}{2}$. And $\text{BAD}_3$ is independent of  $\text{BAD}_1$ and  $\text{BAD}_2$, therefore, $\text{Pr}[\text{BAD}_3|\text{BAD}_1 \wedge \text{BAD}_2]=\frac{1}{2}$. Finally, we have

\begin{align*}
    &\text{Pr}[\text{BAD}_1 \wedge \text{BAD}_2 \wedge \text{BAD}_3]\\
    &\text{Pr}[\text{BAD}_1]\cdot \text{Pr}[\text{BAD}_2|\text{BAD}_1]\cdot \text{Pr}[\text{BAD}_3|\text{BAD}_1 \wedge \text{BAD}_2]\\
    &>\frac{1}{2}\cdot \frac{1}{3}\cdot \frac{1}{2}\\
    &\geq \delta    
\end{align*}

\end{proof}

\begin{corollary}
    For any non-trivial hypothesis class $\mathcal{H}$ with $VCdim\geq 3$, any $0<\epsilon \leq \frac{1}{8}$, $0<\delta <\frac{1}{342}$ and $0<\Delta <\epsilon$, the sample size needed of contrastive learning with accuracy $\epsilon$ in the presence of nasty noise with noise rate $\eta =\frac{1}{2}(\epsilon -\Delta)$ is $\Omega (\frac{\eta}{{\Delta}^2}+\frac{VCdim}{{\Delta}})$.
\end{corollary}

\begin{theorem}
    For an arbitrary distance function $\rho$, a dataset of size $N$, and any noise rate $\eta >0$, the sample complexity of contrastive learning with accuracy $\epsilon =2\eta +\Delta$ in the presence of nasty noise with noise rate $\eta$ is $n(\epsilon ,\delta ,\Delta)=\Omega (\frac{\eta}{\Delta ^2}+\frac{N^2}{\Delta})$.
\end{theorem}

\begin{theorem}
    For the ${\ell}_p$ distance ${\rho}_p$, a dataset of size $N$, and any noise rate $\eta >0$, the sample complexity of contrastive learning with accuracy $\epsilon =2\eta +\Delta$ in the presence of nasty noise with noise rate $\eta$ is $n(\epsilon ,\delta ,\Delta)=\Omega \left (\frac{\eta}{\Delta ^2}+\frac{\text{min}(Nd,N^2)}{\Delta}\right )$.
\end{theorem}

\section{Upper Bound}
\subsection{Upper bound in classical PAC case}
\begin{theorem}[Upper bound for arbitrary distances]
    For an arbitrary distance function $\rho$ and a dataset of size $N$, the sample complexity of contrastive learning is $n(\epsilon,\delta)=O \left (\frac{N^2}{\epsilon}\text{polylog}\left (\frac{1}{\epsilon},\frac{1}{\delta}\right )\right )$ in the classical PAC case.
\end{theorem}

\begin{proof}
Consider any set of samples $\{(x_i,y_i,z_i)\}_{i=1,...,k}$ of size $k\geq N^2$. There exists a data point x such that there are at least $N$ samples which have x as their anchor element. We denote them as $(x,y_{1},z_{1}),...,(x,y_{N},z_{N})$. Consider a graph that has a vertex corresponding to each element in the dataset. Create an undirected edge in this graph between each pairs of vertices $(y_{1},z_{1}),...,(y_{n},z_{n})$.
Since the number of edges is equal to the number of vertices, there must exist a cycle $\mathcal{C}$ in this graph. We can index the vertices along this cycle as $v_1,...,v_t$. Now consider the labeling of the samples $\rho (x,v_1)<\rho (x,v_2)<\rho (x,v_3)<...<\rho (x,v_t)<\rho (x,v_1)$. No distance function can satisfy this labeling and hence not all different labelings of this sample set are possible. Any sample set with size larger than $N^2$ will not be shattered. Thus, the upper bound of VC-dimension is $N^2$. By applying Lemma 2.11, we have the upper bound of sample complexity $n(\epsilon,\delta)=O \left (\frac{N^2}{\epsilon}\text{polylog}\left (\frac{1}{\epsilon},\frac{1}{\delta}\right )\right )$.
\end{proof}

\begin{theorem}[Upper bound for $\ell_p$-distances]
    For integer $p$, a dataset $V$ of size $N$, and the ${\ell}_p$ distance ${\rho}_p:\mathrm{V}\times \mathrm{V} \rightarrow \mathbb{R}$ in a $d$-dimensional space, the sample complexity of contrastive learning in classical PAC case is upper bounded as following: for even $p$, $n(\epsilon,\delta)=O\left (\frac{\text{min}(Nd,N^2)}{\epsilon}\text{polylog}\left (\frac{1}{\epsilon},\frac{1}{\delta}\right )\right )$; for odd $p$, $n(\epsilon,\delta)=O\left (\frac{\text{min}(Nd\text{log}N,N^2)}{\epsilon}\text{polylog}\left (\frac{1}{\epsilon},\frac{1}{\delta}\right )\right )$; for constant $d$, $n(\epsilon,\delta)=O\left (\frac{N}{\epsilon}\text{polylog}\left (\frac{1}{\epsilon},\frac{1}{\delta}\right )\right )$.
\end{theorem}

\begin{proof}
In this proof, we will show the VC-dimension of contrastive learning for $\ell_p$-distances using a dataset $V$ of size $N$. We denote the dimension of representation space as $d$. For the first two cases, we assume $d<n$; in the third case, we consider constant $d$. The upper bound of VC-dimension is $O(N\text{min}(d,N))$ for even $p\geq 2$ and $O(nd\text{log}n)$ for odd $p\geq 1$. To prove this, it is suffices to show that for every set of samples $S={(x_i,y_i,z_i)}_{i=1}^n$ of size $n=\tilde{\Omega}(N\text{min}(d,N))$ there exists labeling of $S$ that cannot be satisfied by any embedding in a d-dimensional ${\ell}_p$-space. By applying Lemma 2.11, we can have upper bounds for sample complexity $n(\epsilon ,\delta )$.

\vspace{5pt}
Our proof uses the following result in algebraic geometry:
\begin{claim}
    Let $m\geq \ell \geq 2$ be integers, and let $P_1,...,P_m$ be real polynomials on $\ell$ variables, each of degree $\leq k$. Let
    \begin{equation*}
        U(P_1,...,P_m)=\{\Vec{x}\in \mathbb{R}^\ell |P_i(\Vec{x})\neq 0\text{ for all }i\in [m]\}
    \end{equation*}
    be the set of points $\Vec{x}\in \mathbb{R}^\ell$ which are non-zero in all polynomials. Then the number of connected components in $U(P_1,...,P_m)$ is at most $(\frac{4ekm}{\ell})^\ell$.
\end{claim}

\paragraph{Upper Bound for Even p}
We denote the embedding of each data point $v\in V$ in $d$-dimensional space as $f(v)=(f_1(v),f_2(v),...,f_d(v))$ where $f_i$ is the $i$-th coordinate of the representation function $f: V\rightarrow \mathbb{R}^d$. Let $\Vec{V}=(f_1(v_1),...,f_d(v_1),...,f_1(v_N),...,f_d(v_N))$. For every sample $(x,y,z)$ where $x,y,z\in V$, we define the polynomial $P_{x,y,z}:\mathbb{R}^{Nd}\rightarrow \mathbb{R}$ as 
\begin{equation*}
    P_{x,y,z}(\Vec{V})=\sum\limits_{j=1}^d(f_j(x)-f_j(y))^p-\sum\limits_{j=1}^d(f_j(x)-f_j(z))^p.
\end{equation*}
The labeling of $(x,y,z)$ is $(x,y^+,z^-)$ if and only if $P_{x,y,z}(\Vec{V})<0$, otherwise $(x,z^+,y^-)$ if and only if $P_{x,y,z}(\Vec{V})>0$. Furthermore, we define $P_i(\Vec{V})=P_{x_i,y_i,z_i}(\Vec{V})$ for $i\in [n]$.

\vspace{5pt}
The polynomial $P_{x,y,z}$ is corresponding to the classifier $h_\rho$ we defined. Now we denote the labeling of sample set $S$ as $h$: for any $i\in [n]$, $(x_i,y_i^+,z_i^-)\in S$, define $h(S)\in \{-1,1\}^n$ where $h_i(S)=1$ if $(x_i,y_i,z_i)$ is labeled as $(x_i,y_i^+,z_i^-)$, otherwise $h_i(S)=-1$ if $(x_i,y_i,z_i)$ is labeled as $(x_i,z_i^+,y_i^-)$. 
\begin{lemma}
    For $h\in \{-1,1\}^n$, define $C_h=\{\Vec{x}\in \mathbb{R}^d|\mathbf{sign}P_i(\Vec{x})=h_i\text{ for all }i\in [n]\}$. Then
    \begin{enumerate}
        \item For distinct $h,h'\in \{-1,1\}^n$ we have $C_h\cap C_{h'}=\emptyset$.
        \item Each $C_h$ is either empty or is a union of  connected components of $U(P_1,...,P_m)$.
        \item Let $S$ be a sample set labeled by $h^*$. Then $C_{h^*}\neq \emptyset$ if and only if there is a mapping $V\rightarrow \mathbb{R}^d$ satisfying all the distance constraints of $h^*(S)$.
    \end{enumerate}
\end{lemma}

\paragraph{Proof of Lemma 4.4}
\begin{proof}
    
\end{proof}

Now we apply above setting and observations to the case of even $p$. By Claim 4.3, there are at most $(\frac{4epn}{Nd})^{Nd}$ connected components in the set $U(P_1,...,P_n)$. If there are two labeling of $S$ are $h^*_1,h^*_2$ such that $h^*_1(S)\neq h^*_2(S)$, then by Lemma 4.4 either the set $C_{h^*_1}$ and $C_{h^*_2}$ are different connected components, or at least one of them is empty. If the size of $S$ is $n$, the number of possible labeling of $S$ is $2^n$, and $2^n>(\frac{4epn}{Nd})^{Nd}$ when $n\geq cNd$ for a sufficiently large constant $c$. Therefore, for at least one labeling $h$ of $S$, it holds that  $C_h=\emptyset$. There is no embedding satisfying the labeling of samples $h$, the upper bound $O(Nd)$ follows.

\paragraph{Upper Bound for Odd p}
Unlike in the case of even $p$, for odd $p$, the distance constraints are comprised of the form $|f_j(x)-f_j(y)|^p$, thus are not polynomial constrains. 

Similar to the case of even $p$, denote each data point $v\in V$ in $d$-dimensional space as $f(v)=(f_1(v),f_2(v),...,f_d(v))$. Let $\Vec{V}=(f_1(v_1),...,f_d(v_1),...,f_1(v_N),...,f_d(v_N))$. For every $j\in[d]$, we fix the ordering of the points by defining a permutation ${\pi}^{(j)}:[n]\rightarrow [n]$ such that $f_j(v_{{\pi}^{(j)}\cdot 1})\leq f_j(v_{{\pi}^{(j)}\cdot 2})\leq ...\leq f_j(v_{{\pi}^{(j)}\cdot N})$, and bound the number of satisfiable labeling w.r.t this ordering. We define ${\sigma}_{x,y}^{(j)}=1$ if $f_j(x)\geq f_j(y)$, and ${\sigma}_{x,y}^{(j)}=-1$ otherwise. Then use $\sigma$ as the order, for any sample $(x,y,z)$, define the polynomial $P_{x,y,z}:\mathbb{R}^{Nd}\rightarrow \mathbb{R}$ as 
\begin{equation*}
    P_{x,y,z}(\Vec{V})=\sum\limits_{j=1}^d{\sigma}_{x,y}^{(j)}(f_j(x)-f_j(y))^p-\sum\limits_{j=1}^d{\sigma}_{x,z}^{(j)}(f_j(x)-f_j(z))^p.
\end{equation*}
Note that $P_{x,y,z}(\Vec{V})<0$ if and only if $(x,y,z)$ is labeled as $(x,y^+,z^-)$ satisfying the selected order ${\pi}^{(1)},...,{\pi}^{(d)}$, otherwise $P_{x,y,z}(\Vec{V})>0$. Furthermore, for any $i\in [n]$, $(x_i,y_i^+,z_i^-)\in S$, we define $P_i(\Vec{V})=P_{x_i,y_i,z_i}(\Vec{V})$; and $h(S)\in \{-1,1\}^n$ where $h_i(S)=1$ if $(x_i,y_i,z_i)$ is labeled as $(x_i,y_i^+,z_i^-)$, otherwise $h_i(S)=-1$ if $(x_i,y_i,z_i)$ is labeled as $(x_i,z_i^+,y_i^-)$. 

\vspace{5pt}
Since in the case of even $p$, there are at most $(\frac{4epn}{Nd})^{Nd}$ connected components in the set $U(P_1,...,P_n)$. Therefore, there are at most $(\frac{4epn}{Nd})^{Nd}$ possible labeling w.r.t the ordering ${\pi}^{(1)},...,{\pi}^{(d)}$. Let $n=\Omega (Nd\text{log}N)$, then we have that $\frac{2^n}{N^{Nd}}>(\frac{4epn}{Nd})^Nd$/ Since there are $(N!)^d<N^{Nd}$ labeling choices of the ordering  ${\pi}^{(1)},...,{\pi}^{(d)}$, there are at most $(\frac{4epn}{Nd})^{Nd}\cdot {N^{Nd}}$ possible labeling for which there exists some order such that the labeling w.r.t an order can satisfy the distance constraints according to an embedding in a d-dimensional ${\ell}_p$-space. Since $(\frac{4epn}{Nd})^{Nd}\cdot {N^{Nd}}<2^n$, there exists a choice of labeling are not satisfiable for any order, which indicates that S is not shattered.

\paragraph{Upper Bound for Constant d}
For constant odd $p>0$, a dataset $V$ of size $N$, and the ${\ell}_p$-norm in a $d$-dimensional space, the VC-dimension of contrastive learning is $O(nd^2)$. This gives optimal bound for constant $d$.

First, assume that $d^2<N$. Then, similar to previous cases, denote each data point $v\in V$ in $d$-dimensional space as $f(v)=(f_1(v),f_2(v),...,f_d(v))$. Let $\Vec{V}=(f_1(v_1),...,f_d(v_1),...,f_1(v_N),...,f_d(v_N))$. Since for odd $p$, ${\ell}_p$-distance function are comprised of the form $|f_j(x)-f_j(y)|^p$, to represent the effects introduced by the absolute value, we use $\tau \in \{-1,1\}^{2d}$ to denote the sign of absolute value. Now we define the polynomial $P_{\tau ,x,y,z}$ as
\begin{equation*}
    P_{\tau ,x,y,z}(\Vec{V})=\sum\limits_{j=1}^d\tau (j)(f_j(x)-f_j(y))^p-\sum\limits_{j=1}^d\tau (d+j)(f_j(x)-f_j(z))^p.
\end{equation*}
where $\tau (j)$ is the value of the $j$'s coordinate of $\tau$. 

There are $2^{2d}$ polynomials associated with each sample $(x,y,z)$. By claim 4.3, there are there are at most $(\frac{4ep2^{2d}n}{Nd})^{Nd}$ connected components in the set $U(P_1,...,P_2^{2d}n)$. Therefore, if take $m\geq cnd^2$ for sufficiently large $c$, then we have $2^n>(\frac{4ep2^{2d}n}{Nd})^{Nd}$, which means that there is a labeling with a choice of signs can not be satisfied, i.e. the set of samples cannot be shattered.
\end{proof}

\subsection{Upper bound in nasty noisy case}
\begin{theorem}
    For any non-trivial hypothesis class $\mathcal{H}$, any $\eta <\frac{1}{2}$, $\delta >0$, $\Delta >0$, the sample size required to PAC learn a classifier in contrastive learning with accuracy $\epsilon=2\eta +\Delta$ and confidence $\delta$ in the presence of nasty noise with noise rate $\eta$ is $O\left (\frac{1}{{\Delta}^2}(VCdim+\text{log}\frac{1}{\delta})\right )$.
\end{theorem}

\begin{proof}
First we assume that, given a sample set of size $n=\frac{c}{{\Delta}^2}(d+\text{log}\frac{2}{\delta})$, with high probability, the adversary modifies the samples at most $n(\eta +\frac{\Delta}{4})$. Denote the number of samples modified by the adversary as $m\sim Bin(n,\eta)$ with expectation $E[m]=\eta n$. By Hoeffding's inequality, we have
\begin{equation*}
    \text{Pr}\left [m>\left (\eta +\frac{\Delta}{4}\right )n\right]\leq \text{exp}\left (-\frac{2n{\Delta}^2}{16}\right )
\end{equation*}
Choose a suitable constant $c$ to make $n\geq \frac{8}{{\Delta}^2}\text{log}\frac{2}{\delta}$, then the probability that this event happens is at most $\frac{\delta}{2}$.

\vspace{5pt}
Consider the process: the original sample set $S'$ is be labeled by a target classifier $h^*$ and the adversary modifies at most $n(\eta +\frac{\Delta}{4})$ samples, then the modified set $S$ is given to the algorithm $A$. With probability at least $1-\frac{\delta}{2}$, Algorithm A will be able to choose a function $h\in \mathcal{H}$ such that $h$ will misclassify at most $n(\eta +\frac{\Delta}{4})$ sample shown to it. However, in the worst case, the mistakes of $h$ occur in points which were not modified. And $h$ may misclassify all samples that the adversary modified. Thus in this case, we are guaranteed that the classifier $h$ errs on no more than $2n(\eta +\frac{\Delta}{4})$ samples on the original sample set $S'$. By Theorem 2.9, with probability at least $1-\frac{\delta}{2}$, there is a constant $c$, the sample set $S'$ of size $\frac{c}{{\Delta}^2}(d+\text{log}\frac{2}{\delta})$ is a $\frac{\Delta}{2}$-sample for the class of symmetric differences between hypothesis $h\in \mathcal{H}$ and the target $h^*$, denoted as $\{h\Delta h^*:h\in \mathcal{H}\}$. Then by the union bound, with probability $1-\delta$, $m\leq n(\eta +\frac{\Delta}{4})$, it follows $|\{(x,y,z))|(x,y,z)\in S', h(x,y,z)\neq h^*(x,y,z)\}|\leq 2n(\eta +\frac{\Delta}{4})$. Therefore, 
\begin{equation*}
    \text{Pr}_\mathcal{D}[h\Delta h^*]\leq 2\eta +\frac{\Delta}{2}<2\eta +\Delta =\epsilon.
\end{equation*}
\end{proof}

\begin{theorem}
    For an arbitrary distance function $\rho$, a dataset of size $N$, and any noise rate $\eta >0$, the sample complexity of contrastive learning with accuracy $\epsilon =2\eta +\Delta$ in the presence of nasty noise with noise rate $\eta$ is $n(\epsilon ,\delta ,\Delta)=O\left (\frac{1}{{\Delta}^2}(N^2+\text{log}\frac{1}{\delta})\right )$.
\end{theorem}

\begin{theorem}
    For the ${\ell}_p$ distance ${\rho}_p$, a dataset of size $N$, and any noise rate $\eta >0$, the sample complexity of contrastive learning with accuracy $\epsilon =2\eta +\Delta$ in the presence of nasty noise with noise rate $\eta$ is upper bounded as following: for even $p$,  $n(\epsilon ,\delta ,\Delta)=O\left (\frac{1}{{\Delta}^2}(\text{min}(Nd,N^2)+\text{log}\frac{1}{\delta})\right )$; for odd $p$,  $n(\epsilon ,\delta ,\Delta)=O\left (\frac{1}{{\Delta}^2}(\text{min}(Nd\text{log}N,N^2)+\text{log}\frac{1}{\delta})\right )$; for constant $d$, $n(\epsilon ,\delta ,\Delta)=O\left (\frac{1}{{\Delta}^2}(N+\text{log}\frac{1}{\delta})\right )$.
\end{theorem}

\section{Data-dependent Sample Complexity}
\paragraph{k negative}
\begin{enumerate}
    \item $L_{con}(f)=\mathbb{E}[\ell(\{\rho(f(x),f(x_i)\}_{i=1}^{k+1})]$
    \item $\hat{L}_{con}(f)=\frac{1}{n}\sum \limits_{j=1}^n\ell(\{\rho(f(x_j),f(x_{ji})\}_{i=1}^{k+1})$
\end{enumerate}

\begin{theorem}
    Assume $\norm{f(\cdot)}_2\leq R$ for any $f\in \mathcal{F}$. Let S be a sample set in the form of $(x,y^+,z^-)$. Let $\ell :\mathbb{R}^k\rightarrow \mathbb{R}_+$ is $L$-lipschitz w.r.t. the ${\ell}_2$-norm. Then with probability at least $1-\delta$ over the training set $S$, for any $f\in \mathcal{F}$
    \begin{equation*}
        L_{con}(f^*)\leq L_{con}(f)+O\left(\frac{L\mathcal{R}_S(\mathcal{F'})}{n}+\sqrt{\frac{\text{log}\frac{1}{\delta}}{n}}\right)
    \end{equation*}
    where
    \begin{equation*}
        \mathcal{R}_S(\mathcal{F'})=\mathop{\mathbb{E}}\limits_{{\sigma \sim \{\pm 1\}}^{2n(k+1)d}}\left [\mathop{sup}\limits_{f\in \mathcal{F}}\sum\limits_{j\in [n]}\sum\limits_{i\in [k+1]} \sum\limits_{t\in [d]} \left (\sigma_{jit1}f(x_j)+\sigma_{jit2}f(x_{ji})\right )\right]
    \end{equation*}
\end{theorem}

\begin{lemma}[Vector contraction lemma]
    Let $\mathcal{X}$ be any set, $(x_1,...,x_n)\in \mathcal{X}^n$, let $\mathcal{F}$ be a class of functions $f:\mathcal{X}\rightarrow {\ell}_2$ and let $h_i:{\ell}_2\rightarrow \mathbb{R}$ have Lipschitz norm $L$. Then
    \begin{equation*}
        \mathbb{E}\mathop{sup}\limits_{f\in \mathcal{F}}\sum \limits_i {\sigma}_i h_i(f(x_i))\leq \sqrt{2}L\mathbb{E}\mathop{sup}\limits_{f\in \mathcal{F}}\sum \limits_{i,k} {\sigma}_{ik} f_k(x_i)).
    \end{equation*}
\end{lemma}

\begin{lemma}[Generalization error bound]
    For a real function class $G$ whose functions map from a set $Z$ to $[0,1]$ and for any $\delta >0$, if $S$ is a training set composed by $n$ iid samples ${z_j}_{j=1}^n$, then with probability at least $1-\frac{\delta}{2}$, for all $g\in G$
    \begin{equation*}
        \mathbb{E}[g(z)]\leq \frac{1}{n}\sum \limits_{j=1}^n g(z_i)+\frac{2\mathcal{R}_S(G)}{n}+3\sqrt{\frac{\text{log}\frac{4}{\delta}}{2n}}.
    \end{equation*}
\end{lemma}

\begin{proof}[Proof of Theorem 5.1]
First, consider the Lipschitz continuity of distance function.
Let $f':V^{2}\rightarrow \mathbb{R}^{2d}$ be defined as
\begin{equation*}
    f'(x,x_i)=(f(x),f(x_i))\in \mathbb{R}^{2d}
\end{equation*}
and $\rho:\mathbb{R}^{2d}\rightarrow \mathbb{R}$ be defined as
\begin{equation*}
    \rho(u,v)=\norm{u-v}_2
\end{equation*}
Then 
\begin{equation*}
    \norm{f(x)-f(x_i)}_2=\rho \circ f'(x,x_i)
\end{equation*}
Clearly, the $\ell_2$-norm distance function is 1-Lipschitz w.r.t. $\ell_2$-norm. Thus, by lemma 5.2, we have
\begin{equation*}
    \mathop{\mathbb{E}}\limits_{{\sigma \sim \{\pm 1\}}^{n}}\left [\mathop{sup}\limits_{f'\in \mathcal{F}'}\sum\limits_{j\in [n]} \sigma_j(\rho \circ f')(x,x_i)\right]
    \leq \sqrt{2}\mathop{\mathbb{E}}\limits_{{\sigma \sim \{\pm 1\}}^{2dn}}\left [\mathop{sup}\limits_{f'\in \mathcal{F}'}\sum\limits_{j\in [n]}\sum\limits_{t\in [d]} \left (\sigma_{jt1}f(x)+\sigma_{jt2}f(x_i)\right )\right]
\end{equation*}

Next, we proof theorem 5.1. According to $L$-Lipschitz of loss function $\ell$ w.r.t. $\ell_2$-norm, applying lemma 5.2, we have
\begin{align*}
    &\mathop{\mathbb{E}}\limits_{{\sigma \sim \{\pm 1\}}^{n}}\left [\mathop{sup}\limits_{f'\in \mathcal{F}'}\sum\limits_{j\in [n]} \sigma_j\ell(\{\rho(f(x_j),f(x_{ji})\}_{i=1}^{k+1})\right]\\
    &\leq \sqrt{2}L\mathop{\mathbb{E}}\limits_{{\sigma \sim \{\pm 1\}}^{n(k+1)}}\left [\mathop{sup}\limits_{f\in \mathcal{F}}\sum\limits_{j\in [n]}\sum\limits_{i\in [k+1]} \sigma_{ji}\norm{f(x_j)-f(x_{ji})}_2\right]\\
    &\leq 2L\mathop{\mathbb{E}}\limits_{{\sigma \sim \{\pm 1\}}^{2n(k+1)d}}\left [\mathop{sup}\limits_{f\in \mathcal{F}}\sum\limits_{j\in [n]}\sum\limits_{i\in [k+1]} \sum\limits_{t\in [d]} \left (\sigma_{jit1}f(x_j)+\sigma_{jit2}f(x_{ji})\right )\right]\\
    &\Rightarrow \mathcal{R}_S(\mathcal{G})\leq 2L\mathcal{R}_S(\mathcal{F'})
\end{align*}
Then by lemma 5.3, 
\begin{align*}
    &L_{con}(f^*)\leq \hat{L}_{con}(f)+\frac{4L\mathcal{R}_S(\mathcal{F'})}{n}+3\sqrt{\frac{\text{log}\frac{4}{\delta}}{2n}}=\hat{L}_{con}(f)+ O\left (\frac{L\mathcal{R}_S(\mathcal{F'})}{n}+\sqrt{\frac{\text{log}\frac{1}{\delta}}{n}}\right )\\
    &\text{where } \mathcal{R}_S(\mathcal{F'})=\mathop{\mathbb{E}}\limits_{{\sigma \sim \{\pm 1\}}^{2n(k+1)d}}\left [\mathop{sup}\limits_{f\in \mathcal{F}}\sum\limits_{j\in [n]}\sum\limits_{i\in [k+1]} \sum\limits_{t\in [d]} \left (\sigma_{jit1}f(x_j)+\sigma_{jit2}f(x_{ji})\right )\right]
\end{align*}
\end{proof}

\begin{corollary}
    If the generalization error $\leq \epsilon$, for $0<\epsilon <1$ and $0<\delta <1$, the data-dependent sample complexity is $n(\epsilon,\delta)=O(kd)$
\end{corollary}

\paragraph{binary case}
\begin{enumerate}
    \item $L_{con}(f)=\mathbb{E}[\ell(h_{\rho}(f(x),f(x_1),f(x_2))]$
    \item $\hat{L}_{con}(f)=\frac{1}{n}\sum \limits_{j=1}^n\ell(\{h_{\rho}(f(x_j),f(x_{j1}),f(x_{j2}))\}_{j=1}^n)$
\end{enumerate}

\begin{proof}[Proof of binary case]
Since sign function is not continuously differentiable, so we assume the hypothesis $h_\rho$ is calculation of difference without sign function. Then we consider the Lipschitz continuity of 
\begin{equation*}
    h_\rho(u,v,w) = \norm{u-v}_2-\norm{u-w}_2
\end{equation*}
By triangle inequality we have
\begin{align*}
    |h_\rho(u_1,v_1,w_1)-h_\rho(u_2,v_2,w_2)| &= |\norm{u_1-v_1}_2-\norm{u_1-w_1}_2-\norm{u_2-v_2}_2+\norm{u_2-w_2}_2|\\
    &\leq |\norm{u_1-v_1}_2-\norm{u_2-v_2}_2| + |\norm{u_1-w_1}_2-\norm{u_2-w_2}_2|\\
    &\leq \norm{u_1-v_1-u_2+v_2}_2 + \norm{u_1-w_1-u_2+w_2}_2\\
    &=\norm{(u_1-u_2)-(v_1-v_2)}_2 + \norm{(u_1-u_2)-(w_1-w_2)}_2\\
    &\leq \norm{u_1-u_2}_2+\norm{v_1-v_2}_2+\norm{u_1-u_2}_2+\norm{w_1-w_2}_2\\
    &=2\norm{u_1-u_2}_2+\norm{v_1-v_2}_2+\norm{w_1-w_2}_2\\
    &\leq 2\norm{(u_1,v_1,w_1)-(u_2,v_2,w_2)}_2
\end{align*}
Therefore, $h_\rho$ is 2-Lipschitz.

According to $L$-Lipschitz of loss function $\ell$ w.r.t. $\ell_2$-norm, applying lemma 5.2, we have
\begin{align*}
    &\mathop{\mathbb{E}}\limits_{{\sigma \sim \{\pm 1\}}^{n}}\left [\mathop{sup}\limits_{f'\in \mathcal{F}'}\sum\limits_{j\in [n]} \sigma_j\ell(h_{\rho}(f(x_j),f(x_{j1}),f(x_{j2})))\right]\\
    &\leq \sqrt{2}L\mathop{\mathbb{E}}\limits_{{\sigma \sim \{\pm 1\}}^{n}}\left [\mathop{sup}\limits_{f'\in \mathcal{F}'}\sum\limits_{j\in [n]} \sigma_j h_{\rho}(f(x_j),f(x_{j1}),f(x_{j2}))\right]\\
    &\leq 4L\mathop{\mathbb{E}}\limits_{{\sigma \sim \{\pm 1\}}^{3nd}}\left [\mathop{sup}\limits_{f'\in \mathcal{F}'}\sum\limits_{j\in [n]} \sum\limits_{t\in [d]}\sigma_{jt0}f(x_j)+\sigma_{jt1}f(x_{j1})+\sigma_{jt2}f(x_{j2})\right]
\end{align*}
\end{proof}

\begin{corollary}
    If the generalization error $\leq \epsilon$, for $0<\epsilon <1$ and $0<\delta <1$, the data-dependent sample complexity is $n(\epsilon,\delta)=O(d)$
\end{corollary}


\begin{thebibliography}{5}
\bibitem{ref1}Alon, Noga, et al. "Optimal sample complexity of contrastive learning." arXiv preprint arXiv:2312.00379 (2023).

\bibitem{ref2}Saunshi, Nikunj, et al. ”A theoretical analysis of contrastive unsupervised representation learning.” International Conference on Machine Learning. PMLR, 2019.

\bibitem{ref3}Lei, Yunwen, et al. ”Generalization analysis for contrastive representation learning.” International Conference on Machine Learning. PMLR, 2023.

\bibitem{ref4}Bshouty, Nader H., Nadav Eiron, and Eyal Kushilevitz. "PAC learning with nasty noise." Theoretical Computer Science 288.2 (2002): 255-275.

\bibitem{ref5}Cesa-Bianchi, Nicolo, et al. "Sample-efficient strategies for learning in the presence of noise." Journal of the ACM (JACM) 46.5 (1999): 684-719.
\end{thebibliography}
\end{document}